\documentclass[letterpaper, 10 pt, journal, twoside]{ieeetran}

\usepackage{graphicx}
\usepackage{subfig}
\usepackage{amsmath}
\usepackage{amssymb}
\usepackage[font=footnotesize]{caption} 
\usepackage{subfig} 
\usepackage[noadjust]{cite} 
\usepackage{color}
\usepackage{algorithm} 
\usepackage{algpseudocode} 
\usepackage{enumerate}
\usepackage[hidelinks,colorlinks=false]{hyperref}
\usepackage[left=0.75in, right=0.75in, top=0.75in, bottom=0.75in]{geometry}
\usepackage{authblk} 
\usepackage{arydshln} 

\usepackage{bm}
\usepackage{tikz}
\usetikzlibrary{calc} 
\usetikzlibrary{shapes} 
\usetikzlibrary{chains}
\usetikzlibrary{fit}
\usetikzlibrary{arrows}
\usetikzlibrary{decorations.text} 
\usetikzlibrary{decorations.markings}
\usetikzlibrary{decorations.pathmorphing} 
\usetikzlibrary{shadows}
\usetikzlibrary{patterns}
\usetikzlibrary{matrix}
\usepackage{pgfplots}
\usepackage[europeanresistors]{circuitikz}
\usepackage[outline]{contour} 
\contourlength{1.5pt}
\usetikzlibrary{fadings}
\usetikzlibrary{patterns}
\usetikzlibrary{shapes}

\newtheorem{theorem}{Theorem}
\newenvironment{proof}{{\indent \indent \it Proof:}}{\hfill $\square$ \vspace{0.5em}}
\newtheorem{lemma}{Lemma}

\newcommand{\e}{\mathrm{e}}

\graphicspath{{figures/}}

\title{Prespecified-Performance Kinematic Tracking Control for Aerial Manipulation}
\author{Huazi~Cao, Jiahao~Shen, Zhengzhen~Li, Qinquan~Ren, Shiyu~Zhao
\thanks{Manuscript received: May 3, 2025; Accepted August 26, 2025. This paper was recommended for publication by Editor Cosimo Della Santina upon evaluation of the Associate Editor and
Reviewers’ comments. This work was supported by the National Major Research \& Development Plan - Intelligent Robotics Major Special Project (Grant No. 2023YFB4705500), the National Natural Science Foundation of China (Grant No. 62473320),  and the Open Foundation of the National Key Laboratory of Autonomous Intelligent Unmanned Systems (Grant No. 2024-SRIAS-KF-007). (Corresponding author: Shiyu Zhao)}
\thanks{Huazi Cao is with the Westlake Institute for Optoelectronics and the WINDY Lab in the Department of Artificial Intelligence, Westlake University, Hangzhou 311421, China (E-mail: caohuazi@wioe.westlake.edu.cn)}
\thanks{Jiahao Shen, Zhengzhen Li, and Shiyu Zhao are with the WINDY Lab in the Department of Artificial Intelligence, Westlake University, Hangzhou 310024, China (E-mail: \{shenjiahao, lizhengzhen, zhaoshiyu\}@westlake.edu.cn).}
\thanks{Qinyuan Ren is with the College of Control Science and Engineering, Zhejiang University, Hangzhou 310027, China (E-mail: renqinyuan@zju.edu.cn). }
}

\begin{document}
	
\maketitle
\begin{abstract}
        This paper studies the kinematic tracking control problem for aerial manipulators. Existing kinematic tracking control methods, which typically employ proportional-derivative feedback or tracking-error-based feedback strategies, may fail to achieve tracking objectives within specified time constraints. To address this limitation, we propose a novel control framework comprising two key components: end-effector tracking control based on a user-defined preset trajectory and quadratic programming-based reference allocation. Compared with state-of-the-art approaches, the proposed method has several attractive features. First, it ensures that the end-effector reaches the desired position within a preset time while keeping the tracking error within a performance envelope that reflects task requirements. Second, quadratic programming is employed to allocate the references of the quadcopter base and the Delta arm, while considering the physical constraints of the aerial manipulator, thus preventing solutions that may violate physical limitations. The proposed approach is validated through three experiments. Experimental results demonstrate the effectiveness of the proposed algorithm and its capability to guarantee that the target position is reached within the preset time. 
\end{abstract}


\section{Introduction}

An aerial manipulator, as a composite robotic system that integrates a multirotor base with a manipulator, not only inherits the three-dimensional maneuverability of the multicopter but also extends its physical interaction capabilities through the end-effector. This unique combination of autonomous flight and precise manipulation has attracted increasing research attention. In recent years, numerous groundbreaking advances have emerged (see \cite{Ollero2021,Xilun2019A,ruggiero2018aerial} for recent surveys), with typical applications covering aerial grasping \cite{luo2023time,cao2024motion}, contact-based inspection \cite{bodie2020active,guo2024aerial}, peg-in-hole operations \cite{lee2023virtual,wang2023millimeter} (Fig~\ref{fig_pig_in_hole}), etc. These applications have progressively validated the practical value of the aerial manipulator.

Developing high-accuracy control methods is essential for aerial manipulators. Current control methods for aerial manipulators primarily focus on low-level motion control and kinematic control. Low-level motion control methods leverage the aerial manipulator's dynamics to design the control law that regulates rotor speeds and manipulator joint torques. Existing low-level motion control methods can be categorized into decoupled control \cite{fumagalli2014developing,cao2023eso} and coupled control \cite{yang2014dynamics,welde2021dynamically}. However, low-level motion control alone is not sufficient for millimeter-level precision owing to model inaccuracies and disturbance susceptibility \cite{wang2023millimeter}. This motivates the exploration of alternative approaches, such as kinematic control, for aerial manipulators.

\begin{figure}[t]
	\centering
	\includegraphics[width=1.0\linewidth]{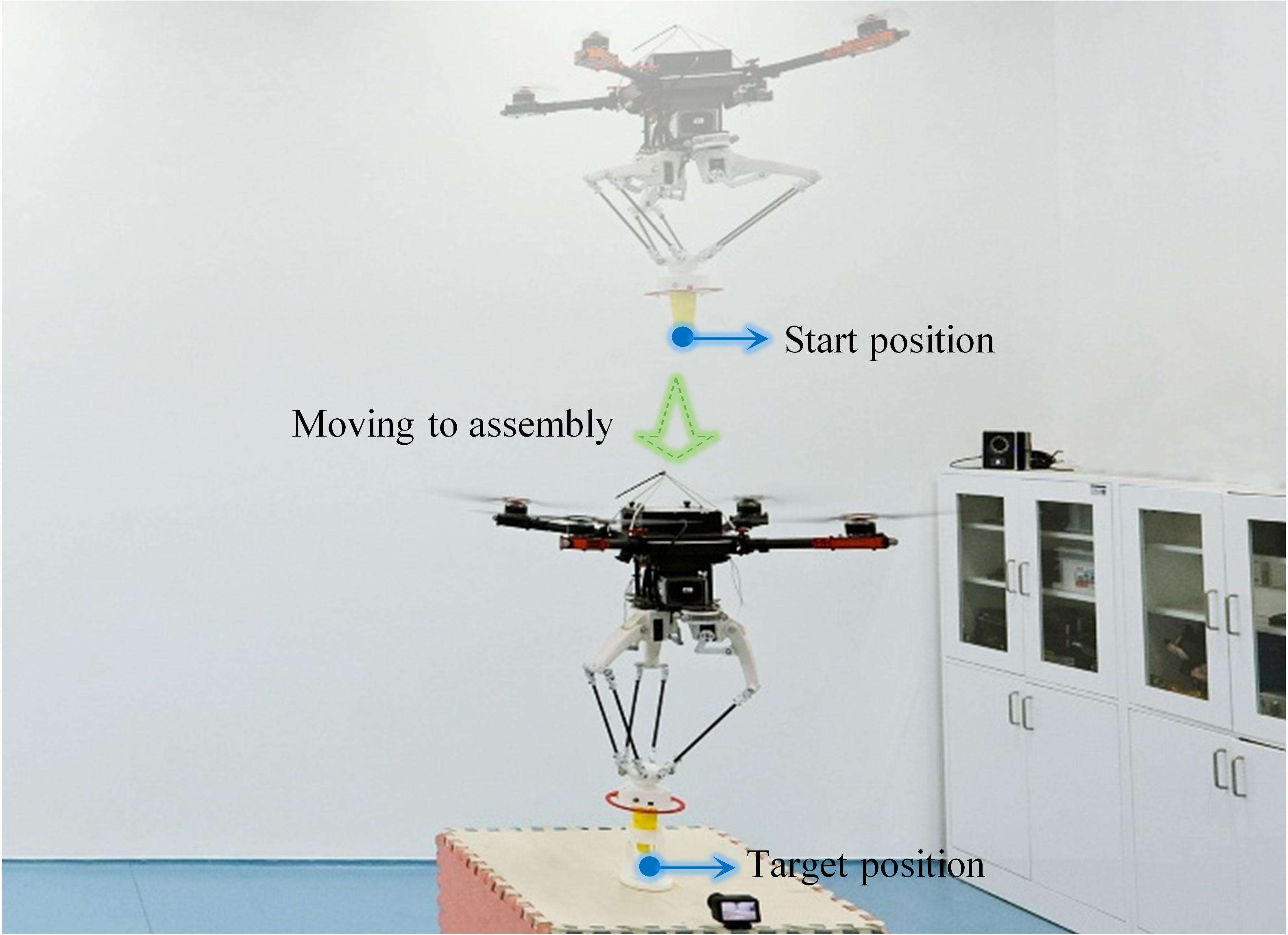}
	\caption{Aerial peg-in-hole assembly by an aerial manipulator. Details are available in the supplementary video.}
	\label{fig_pig_in_hole}
\end{figure}

Kinematic control is derived from the kinematics of aerial manipulators, focusing on achieving high-precision end-effector control through the coordination of motions of the multirotor base and the manipulator. Compared to ground-based manipulators, the kinematic control of aerial manipulators is more challenging for two reasons. First, the manipulator and the multirotor base are dynamically coupled, indicating that their motions have a mutual influence on each other. Second, aerial manipulators operate within a 3D environment rather than in a 2D environment. Moreover, most tasks necessitate 3D coordinated motion between the multirotor base and the manipulator for successful execution. Existing studies on kinematic control methods for aerial manipulators primarily focus on two aspects.

The first aspect is the control architecture. Current control architectures can be classified into two types: decoupled architecture and coupled architecture.  The decoupled architecture separately designs kinematic control for the multirotor base and the manipulator \cite{wang2023millimeter,chermprayong2019integrated}. It achieves end-effector control through a sequential ``platform positioning first, manipulator operation second'' strategy \cite{cao2024motion}. Although simple to implement, this approach suffers from low energy efficiency and slow response speed \cite{Ollero2021}. In contrast, coupled architecture treats the multirotor base and the manipulator as an integrated system \cite{cao2023eso}, coordinating their movements through Jacobian matrix-based methods. Although this approach demonstrates superior energy efficiency and enhanced response characteristics, it introduces significant control complexity due to the necessity of simultaneously managing the degrees of freedom of both subsystems.

The second aspect is the feedback strategy. Since higher-order motion information of both targets and end-effectors is generally difficult to obtain, positional feedback is typically employed to implement kinematic control of aerial manipulators \cite{sanchez2015integral}. There are two existing feedback strategies. The first is the closed-loop inverse kinematics (CLIK) strategy, which utilizes a control law analogous to Proportional-Derivative control to achieve end-effector tracking \cite{sciavicco2010robotics,chiacchio1991closed,muscio2017coordinated}. The second strategy directly employs end-effector tracking error, typically implemented in the form of model predictive control (MPC) \cite{lunni2017nonlinear,lee2020aerial}. By incorporating the end-effector tracking error into the MPC cost function, this strategy generates references for both the multirotor base and the manipulator that reduce the tracking error \cite{luo2023time}. Although both feedback strategies can accomplish end-effector tracking, neither guarantees that the end-effector reaches the target point within the specified time.

The objective of this paper is to achieve kinematic control that enables the end-effector to complete tracking tasks within specified time constraints. The above analysis reveals the limitations in the two aspects of the existing kinematic control approaches for aerial manipulation. This paper aims to overcome these limitations by proposing a new framework that incorporates preset trajectory tracking control and quadratic programming (QP)-based trajectory allocation (see Fig.~\ref{fig_structureController}(a)). The proposed algorithms are verified by carefully designed experiments on an aerial manipulation platform. The novelty of the proposed algorithms is summarized below.

1) We propose a novel two-layer architecture to achieve the tracking control of the end-effector with prespecified performance. The architecture enables the aerial manipulator's end-effector to reach the designated position within a preset time. Moreover, by incorporating the aerial manipulator's physical constraints into the reference trajectory allocation, it prevents the aerial manipulator from oscillating at constraint boundaries. Compared with existing methods based on CLIK \cite{cataldi2019set} and tracking-error feedback strategies \cite{lunni2017nonlinear}, the proposed method guarantees task completion time and ensures better compliance with planned execution outcomes. 

2) We propose a method for generating performance envelope boundaries according to the task completion metrics, including the initial tracking error, steady-state tracking error requirements, and convergence time. This approach enables researchers to flexibly design the parameters of performance envelope boundaries according to actual task requirements.

3) We employ a preset tracking error trajectory \cite{shi2024apreset} to constrain the end-effector's tracking error. Compared with existing prescribed performance control methods based on barrier functions \cite{chen2024adaptive}, our approach effectively avoids singularity issues. In conventional methods, barrier functions are prone to generating singular phenomena under boundary conditions, which not only leads to divergence of the control law but may also cause system safety hazards \cite{shi2024preset}. In contrast, the method in this work fundamentally eliminates the possibility of singularity occurrence, thereby significantly enhancing system safety and reliability.
	
\section{Modeling}
This section shows coordinate frames, notations, and kinematics of the aerial manipulator. It develops a kinematic model for end-effector tracking of the aerial manipulator.

\begin{figure*}[t]
	\centering
	\includegraphics[width=1.0\linewidth]{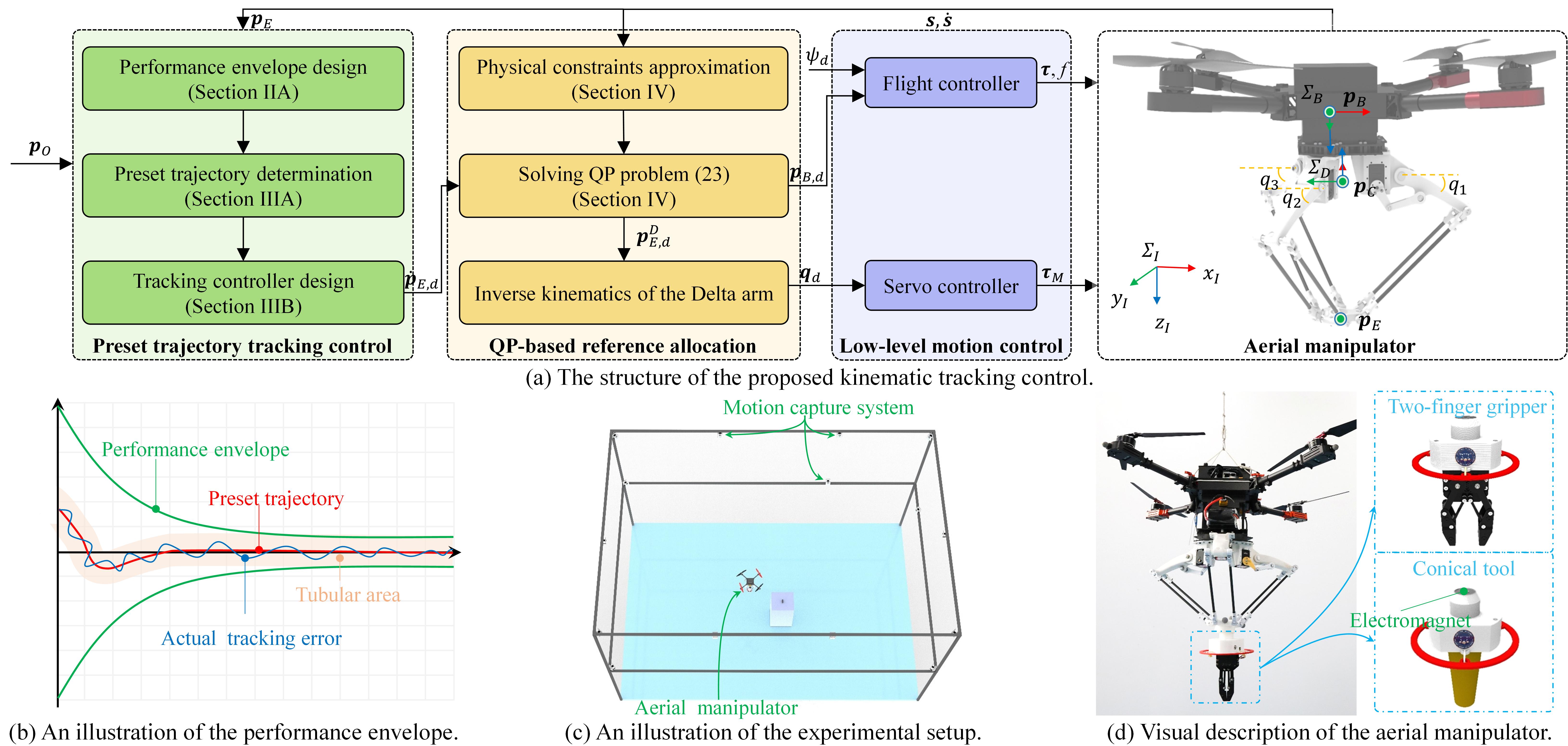}
	\caption{ The system overview, the preset envelope, and experimental setup of the aerial manipulator system.}
	\label{fig_structureController}
\end{figure*}

\subsection{Coordinate Frames and Notations} \label{sec_problem_set}

The aerial manipulator has three reference frames: the inertial frame $\varSigma_I$, the quadcopter body-fixed frame $\varSigma_B$, and the Delta arm frame $\varSigma_D$ (see Fig.~\ref{fig_structureController}(a)). $\varSigma_I$ is an inertial frame where the $z$-axis is in the direction of the gravity vector. $\varSigma_B$ is rigidly attached to the quadcopter base. Its origin coincides with the center of gravity of the quadcopter.  $\varSigma_D$ is rigidly attached to the Delta arm base at its geometric center $\bm{p}_C\in\mathbb{R}^3$. 

Denote the position of the end-effector in $\varSigma_I$ as $\bm{p}_E\in\mathbb{R}^3$. Let $\bm{p}_{O}\in\mathbb{R}^3$ denote the object position in $\varSigma_I$. Then, the tracking error of the end-effector is $\bm{e}_E = \bm{p}_E - \bm{p}_{O}=[e_{E,x},e_{E,y} ,e_{E,z}]^T \in \mathbb{R}^3$. To clearly represent the Cartesian components of vectors, we define $\nu\in \{x, y, z\}$ as the component index. For instance, the components of the tracking error vector $\bm{e}_{e}$ can be expressed as $e_{E,\nu}$, where $\nu\in \{x, y, z\}$.

\subsection{Kinematics of the Aerial Manipulator}
The aerial manipulator analyzed in this paper comprises a quadcopter base and a Delta arm (see Fig.~\ref{fig_structureController}(a)). The Delta arm's end-effector position can translate relative to the base, while its orientation remains aligned with the base orientation. Denote the position of the quadcopter base in $\varSigma_I$ as $\bm{p}_B\in\mathbb{R}^3$. Let $\bm{R}_B\in SO(3)$ denote the rotation matrix from $\varSigma_B$ to $\varSigma_I$. The kinematic expression for the aerial manipulator can then be written as $\bm{p}_E=\bm{p}_B+\bm{R}_B\bm{p}_E^B,$ where, $\bm{p}_E^B \in\mathbb{R}^3$ represents the position of the end-effector in $\varSigma_B$, which can be obtained through the forward kinematics of the manipulator. 

Let $\bm{p}_C^B \in \mathbb{R}^3$ denote the position of the center of the base in $\varSigma_B$. Let $\bm{p}_E^D\in\mathbb{R}^3$ denote the position of the end-effector in $\varSigma_D$. The relationship between $\bm{p}_E^B$ and $\bm{p}_E^D$ is $\bm{p}_E^B = \bm{R}_D^B\bm{p}_E^D+\bm{p}_C^B,$ where $\bm{R}_D^B\in SO(3)$ is the rotation matrix from $\varSigma_D$ to $\varSigma_B$. $\bm{p}_E^D$ is a function of the joint angle vector $\bm{q} = [q_1, q_2, q_3]^T$, and the function can be found in \cite{lopez2006delta}.

As can be seen from Fig.~\ref{fig_structureController}(a), the joint angles of the Delta arm are driven by planar four-bar linkages. The relationship between the joint angles and the crank position angles can be calculated by the kinematics of the planar four-bar linkage \cite[Section~3.6]{shigley2004standard}.

\section{Preset Trajectory Tracking Control} \label{sec_pre_control}
This section proposes a preset trajectory kinematic tracking control method. It ensures that the end-effector reaches the desired position within a preset time while keeping the tracking error within a performance envelope.

\subsection{Performance Envelope}

 The performance envelope is introduced to represent the performance criterion of the aerial manipulator. The boundary function of the performance envelope is denoted as $\bm{\rho}(t)\in\mathbb{R}^3$. According to \cite{shi2024preset}, we employ a boundary function with an exponential form to ensure that the absolute value of the tracking error gradually decreases. The boundary function is designed to $\rho_{\nu}(t)=(\rho_{\nu,0}-\rho_{\nu,\infty})\exp(-l_E t)+\rho_{\nu,\infty} \in\mathbb{R}$, where $\nu\in \{x, y, z\}$. Here, $\bm{\rho}_{0} = [\rho_{x,0},\rho_{y,0},\rho_{z,0}]^T \in \mathbb{R}^3$, $\bm{\rho}_{\infty} = [\rho_{x,\infty},\rho_{y,\infty},\rho_{z,\infty}]^T \in \mathbb{R}^3$, and $l_E$ are positive parameters of the boundary function. 
 
The parameters of the boundary function are determined by analyzing task completion metrics, including the initial tracking error, steady-state tracking error requirements, and convergence time. Specifically, $\bm{\rho}_{0}$ relates to the initial tracking error, requiring that $\rho_{\nu,0}>|{e}_{E, \nu}(0)|$, where $\nu\in \{x, y, z\}$, the definition of ${e}_{E, \nu}$ can be seen in Section~\ref{sec_problem_set}. The parameter $\bm{\rho}_{\infty}$ represents the maximum allowable steady-state tracking error. Its lower bound is typically determined by the end-effector's achievable steady-state tracking accuracy, while the upper bound is determined by the operational precision requirements of the actual tasks. The parameter $l_E$ is the decay rate of the performance function. It can be determined in two steps. The first is to determine a tolerance threshold for the preset trajectory. The exponential decay process of $\bm{\rho}(t)$ exhibits a monotonic decrease toward its asymptotic steady-state value $\bm{\rho}_{\infty}$. The preset time $t_p\in\mathbb{R}$ is defined as the time required for the tracking error to converge to within $\pm\bm{\rho}_{\infty}$. It is constrained by the maximum speed of the aerial manipulator's end-effector. We let $\epsilon_p \in \mathbb{R}$ as the tolerance threshold, and its definition is
 \begin{equation}
     \epsilon_p = (\|\bm{\rho}_{0}\| - \|\bm{\rho}_{\infty}\|)e^{-l_E t_p}.
     \label{eq_epsilon}
 \end{equation}
In practical implementations, $\epsilon_p$ can be determined by a scalar multiple of $\|\bm{\rho}_{\infty}\|$, such as $0.1\|\bm{\rho}_{\infty}\|$, to align with task completion metrics. The second is to calculate the parameter $l_E$ with the value of $\epsilon_p$. According to \eqref{eq_epsilon}, we have $l_E = {\ln\left[(\|\bm{\rho}_{0}\| - \|\bm{\rho}_{\infty}\|]/{\epsilon_p}\right)}/{t_p}$.

\subsection{Preset Tracking Error Trajectory}
The preset tracking error trajectory of the end-effector is introduced to constrain the end-effector’s tracking error. Through the tracking controller in Section~\ref{sec_tracking_controller}, the actual tracking error is constrained to track the preset trajectory, guaranteeing its strict containment within the performance envelope (see Fig.~\ref{fig_structureController}(b)). Let $\bm{\alpha}(t)=[\alpha_x(t),\alpha_y(t),\alpha_z(t)]^T\in\mathbb{R}^3$ denote the preset trajectory. Let $\bm{c}=[c_x,c_y,c_z]^T\in\mathbb{R}^3$ denote the design parameter vector of $\bm{\alpha}(t)$. To ensure the preset trajectory is monotonically decreasing and contained within the performance envelope, it is designed as $\dot{\alpha}_{\nu}(t)=-l_E\alpha_{\nu}(t)+b_{\nu}e^{-(l_E+c_{\nu})t},\alpha_{\nu}(0)=e_{E, \nu}(0), \nu\in\{ x,y,z\}$, where $c_{\nu}>0$ is determined by \eqref{eq_c_nu}, and $b_{\nu}=l_E e_{E,\nu}(0)+\dot{e}_{E,\nu}(0) \in\mathbb{R}$. The solution of $\alpha_{\nu}(t)$ is 
\begin{equation}
\begin{split}
    \alpha_{\nu}(t) = \frac{b_{\nu}}{c_{\nu}}(1-e^{-c_{\nu}t})e^{-l_Et} + e_{E, \nu}(0)e^{-l_Et}.
\end{split}
	\label{eq_alpha}
\end{equation}
From \eqref{eq_alpha}, we have $\alpha_{\nu}(0) = e_{E,\nu}(0)$ and $\alpha_{\nu}(\infty) = 0$. The preset trajectory is initialized with the initial tracking error and converges to zero at its terminal value. The first and second time derivatives of $\alpha_{\nu}(t)$ are 
\begin{equation}
\begin{split}
    \dot{\alpha}_{\nu}(t) = &\frac{b_{\nu}}{c_{\nu}}[-l_Ee^{-l_Et} + (c_{\nu}+l_E)e^{-(c_{\nu}+l_E)t}]\\
    &-e_{E, \nu}(0)l_Ee^{-l_Et},
\end{split}
\end{equation}
\begin{equation}
\begin{split}
    \ddot{\alpha}_{\nu}(t) = &\frac{b_{\nu}}{c_{\nu}}[l_E^2e^{-l_Et} -(c_{\nu}+l_E)^2e^{-(c_{\nu}+l_E)t}]\\
    & +  e_{E, \nu}(0)l_E^2e^{-l_Et}.
\end{split}
\end{equation}

\begin{lemma}[\cite{shi2024apreset}]\label{lemma_performance}
	Let $\varepsilon_{\nu}$ be a positive constant satisfying $\varepsilon_{\nu} <\min\{\rho_{\nu,\infty}, \rho_{\nu,0}-|e_{E,\nu}(0)|\}$, where $\nu\in\{ x,y,z\}$. For the preset trajectory of the tracking error calculated by \eqref{eq_alpha}, if $c_{\nu}>b_{\nu}/(\rho_{\nu,0}-\varepsilon_{\nu} -|e_{E, \nu}(0)|)$ such that $|\alpha_{\nu}(t)| < \rho_{\nu}(t)-\varepsilon_{\nu}, \forall t\geq0$.
\end{lemma}

\begin{proof}
	Define $\gamma_{\nu}(t)=\rho_{\nu}(t)-\varepsilon_{\nu} -|\alpha_{\nu}(t)|$, we have 
	\begin{equation}
		\begin{split}
			\gamma_{\nu}(t) =& (\rho_{\nu,0}-\rho_{\nu,\infty})e^{-l_Et}+\rho_{\nu,\infty} - \varepsilon_{\nu} \\
			& -\left| e_{E, \nu}(0) e^{-l_Et}+\frac{b_{\nu}}{c_{\nu}}[1-e^{-c_{\nu}t}]e^{-l_Et} \right| \\
			\geq& \left[ \rho_{\nu,0} -\varepsilon_{\nu} - |e_{E,\nu}(0)|  - \frac{|b_{\nu}|}{c_{\nu}}(1-e^{-c_{\nu}t})\right] e^{-l_E t} \\
			& +\rho_{\nu,\infty}(1-e^{-l_E t}) - \varepsilon_{\nu} + \varepsilon_{\nu}e^{-l_E t}.\\
			=& \left[ \rho_{\nu,0} -\varepsilon_{\nu} - |e_{E, \nu}(0)|  - \frac{|b_{\nu}|}{c_{\nu}}(1-e^{-c_{\nu}t})\right] e^{-l_E t} \\
			& +(\rho_{\nu,\infty}- \varepsilon_{\nu} )(1-e^{-l_E t}).\\
		\end{split}
		\label{eq_y_E}
	\end{equation}
  	From \eqref{eq_y_E}, one can conclude that if $\rho_{\nu,0}- |e_{E, \nu}(0)| - |b_{\nu}|/c_{\nu}>0$ and $\rho_{\nu,\infty} > \varepsilon_{\nu} $, then $\gamma_{\nu}>0$. According to the definition of $\varepsilon_{\nu}$, $\rho_{\nu,\infty} > \varepsilon_{\nu} $ is naturally satisfied. Therefore, if we let $c_{\nu}>|b_{\nu}|/(\rho_{\nu,0}-\varepsilon_{\nu}-|e_{E, \nu}(0)|)$, then we have $\gamma_{\nu}>0$. From the definition of $\gamma_{\nu}(t)$, we have $|\alpha_{\nu}(t)| < \rho_{\nu}(t)-\varepsilon_{\nu}, \forall t\geq0$. This completes the proof.
\end{proof}

Lemma~\ref{lemma_performance} gives a method to determine the parameter vector $\bm{c}$. Through the regulation of $\bm{c}$, the descent rate of the preset trajectory is made faster than that of the performance envelope boundary function, thereby guaranteeing the tracking error contained within the performance envelope.

\subsection{Tracking Control Design}\label{sec_tracking_controller}

The goal of tracking control is to compute the desired end-effector velocity command $\dot{\bm{p}}_{E,d}$ such that the end-effector position $\bm{p}_E$ tracks the target position $\bm{p}_O$. In order to ensure that the tracking error satisfies the requirements of the performance envelope, we introduce the preset trajectory for the tracking error into the tracking control design (see Fig.~\ref{fig_structureController}b). This ensures that the end-effector of the manipulator can track $\bm{p}_O$ within the preset time.

Let $\bm{z}=\bm{e}_E - \bm{\alpha}(t)$. The sliding mode vector is defined as
\begin{equation}
	\bm{s}=\bm{z}+\bm{\Lambda}\int_{0}^{t}\bm{z}dt,
	\label{eq_slide_mode}
\end{equation}
where $\bm{\Lambda}$ is a positive diagonal matrix. Then, we have 
\begin{equation}
		\dot{\bm{s}}=\dot{\bm{z}}+\bm{\Lambda}\bm{z}=\dot{\bm{p}}_E -\dot{\bm{p}}_O-\dot{\bm{\alpha}}+\bm{\Lambda}\bm{z}.
 \label{eq_s_dot}
\end{equation}
The relationship between the actual velocity and the desired velocity input given to the controller can be described as
\begin{equation}
    \dot{\bm{p}}_E = \dot{\bm{p}}_{E,d} + \bm{\Delta}, \label{eq_pe_dot_model}
\end{equation}
where $\bm{\Delta}$ represents an unknown item. Equation \eqref{eq_pe_dot_model} reflects the capability of reference allocation of the aerial manipulator, since the objective of reference allocation is to ensure that $\dot{\bm{p}}_E$ accurately tracks $\dot{\bm{p}}_{E,d}$. Substituting \eqref{eq_pe_dot_model} into \eqref{eq_s_dot} yields 
\begin{equation}
    \dot{\bm{s}} = \dot{\bm{p}}_{E,d} + \bm{\Delta} -\dot{\bm{p}}_O-\dot{\bm{\alpha}}+\bm{\Lambda}\bm{z}.
\end{equation}
The tracking control law is designed as 
\begin{equation}
	\dot{\bm{p}}_{E,d} =  \dot{\bm{p}}_O + \dot{\bm{\alpha}} -\bm{\Lambda}\bm{z} -\bm{K}\bm{s},
	\label{eq_control_preset}
\end{equation}
where $\bm{K}$ is the gain matrix, and it is a positive diagonal matrix.

\begin{theorem}\label{theorem_posi}
	Assuming that $\bm{\Delta}$ is bounded, i.e., $\|\bm{\Delta}\|\leq \delta_E$, if the tracking control law is designed as \eqref{eq_control_preset}, then the position error $\bm{e}_E$ is bounded and $|e_{E, \nu}|< \rho_{\nu}(t), \forall t \geq 0$.
\end{theorem}
\begin{proof}
	A candidate Lyapunov function is given as $V_{E}(\bm{s})=\frac{1}{2}\bm{s}^T\bm{s}$. The time derivative of $V_E$ is given as 
	\begin{equation}
		\begin{split}
			\dot{V}_E=&\bm{s}^T(\dot{\bm{p}}_E -\dot{\bm{p}}_O-\dot{\bm{\alpha}}+\bm{\Lambda}\bm{z}).\\
			=&\bm{s}^T(\dot{\bm{p}}_{E,d}-\bm{\Delta} -\dot{\bm{p}}_O-\dot{\bm{\alpha}}+\bm{\Lambda}\bm{z}).\\
		\end{split}
		\label{eq_v_dot1} 
	\end{equation} 
	Substituting \eqref{eq_control_preset} into \eqref{eq_v_dot1} yields $\dot{V}_E= -\bm{s}^T\bm{K}\bm{s} - \bm{s}^T{\bm{\Delta}}$. Using the comparison theorem \cite[Section~9.3]{khalil2002nonlinear}, we define $W=\sqrt{V_E}=\left\|\bm{s}\right\|/\sqrt{2}$. Then, we have 
	\begin{equation}
		\begin{split}
			\dot{W}&=\frac{-\bm{s}^T\bm{K}\bm{s}-\bm{s}^T{\bm{\Delta}}}{\sqrt{2}\|\bm{s}\| }\\
			&\leq \frac{-\lambda_{\min}(\bm{K})\|\bm{s}\|^2+\delta_E\left\|\bm{s}\right\| }{\sqrt{2}\left\|\bm{s}\right\| }\\
			&=-\lambda_{\min}(\bm{K})W+{\delta_E}/{\sqrt{2}},
		\end{split}
	\end{equation}
	where $\lambda_{\min}(\bm{K})$ is the minimum eigenvalue of $\bm{K}$. Since $\bm{K}_v$ is a positive diagonal matrix, we have $\lambda_{\min}(\bm{K})>0$. Then, we have 
\begin{equation}
     W(t) \leq W(0)e^{-\lambda_{\min}(\bm{K})t} + \delta_E/(\sqrt{2}\lambda_{\min}(\bm{K})).
\end{equation}
Since $W(0) = 0$, one can obtain 
\begin{equation}
    W(t)\leq \delta_E/(\sqrt{2}\lambda_{\min}(\bm{K})).
\end{equation}
	According to the definition of $W$, we have $\|\bm{s}\|\leq \delta_E/\lambda_{\min}(\bm{K})$. 
	
	Let $\bm{z}_{\int}= \int_{0}^{t}\bm{z}dt$ denote the integral of $\bm{z}$ with respect to time. According to \eqref{eq_slide_mode}, we have 
 \begin{equation}
    \dot{\bm{z}}_{\int} = -\bm{\Lambda}\bm{z}_{\int} + \bm{s}.
    \label{eq_zint}
\end{equation}
 Combing \eqref{eq_zint} and $\bm{z}_{\int}(0) = 0$ yields
	\begin{equation}
		\begin{split}
			\bm{z}_{\int}=&\bm{z}_{\int}(0)e^{-\bm{\Lambda}t} + \int_{0}^{t}e^{-\bm{\Lambda}(t-\tau)}\bm{s}(\tau)d\tau\\
			=& \int_{0}^{t}e^{-\bm{\Lambda}(t-\tau)}\bm{s}(\tau)d\tau
		\end{split}
	\end{equation}
	Therefore, we have 
	\begin{equation}
		\begin{split}
			\|\bm{z}_{\int}\|\leq &\frac{\delta_E}{\lambda_{\min}(\bm{K})} \int_{0}^{t}e^{-\lambda_{\min}(\bm{\Lambda})(t-\tau)}d\tau\\
			=& {\delta_E}(1-e^{-\lambda_{\min}(\bm{\Lambda})t})/({\lambda_{\min}(\bm{K})\lambda_{\min}(\bm{\Lambda})})\\
			\leq & {\delta_E}/({\lambda_{\min}(\bm{K})\lambda_{\min}(\bm{\Lambda})}).
		\end{split}
        \label{eq_z_inf}
	\end{equation}
	According to the definition of $\bm{s}$, one can obtain that   
	\begin{equation}
 \begin{split}
     \| \bm{z}\|-\|\bm{\Lambda}\bm{z}_{\int} \| & \leq | 	\| \bm{z}\|-\|\bm{\Lambda}\bm{z}_{\int}  \|| \leq \|\bm{z}+\bm{\Lambda}\bm{z}_{\int} \|\\
     &=\|\bm{s}\|\leq \delta_E/\lambda_{\min}(\bm{K}).
 \end{split}
 \label{eq_z_upper}
	\end{equation}
Let $\delta_{\bm{z}}$ denote the upper bound of $\|\bm{z}\|$. From \eqref{eq_z_inf} and \eqref{eq_z_upper}, one can conclude that 
 \begin{equation}
 \begin{split}
     \|\bm{z}\|&\leq \|\bm{\Lambda}\bm{z}_{\int}\| +\delta_E/\lambda_{\min}(\bm{K}) \\
     &\leq \frac{(\lambda_{\min}(\bm{\Lambda})+\lambda_{\max}(\bm{\Lambda}))\delta_E}{\lambda_{\min}(\bm{\Lambda})\lambda_{\min}(\bm{K})} = \delta_{\bm{z}}.
 \end{split}
 \end{equation}
 
	Further, we prove that $|e_{E,\nu}|\leq\rho_{\nu}(t), \nu\in\{x,y,z\}$ by two steps. First, the parameter $c_{\nu}$ is determined according to Lemma~\ref{lemma_performance}. The gain matrix $\bm{K}$ is chosen such that $\delta_{\bm{z}}< \min\{\rho_{\nu,\infty}, \rho_{\nu,0}-|\e_{E, \nu}(0)| \}$. The parameters $c_{\nu}$ is chosen as 
	\begin{equation}
		c_{\nu}>b_{\nu}/(\rho_{\nu,0}- \delta_{\bm{z}} - |e_{E, \nu}(0)|), \nu\in\{x,y,z\}.
		\label{eq_c_nu}
	\end{equation}
	From Lemma~\ref{lemma_performance}, we have 
	\begin{equation}
		|\alpha_{\nu}(t)| <  \rho_{\nu}(t) -  \delta_{\bm{z}}, \forall t\geq0 \text{ and } \nu\in\{x,y,z\}.
		\label{eq_alpha_rho}
	\end{equation}
 
Second, the upper bound of $|e_{E, \nu}|$ is obtained from its definition and the upper bound of $|\alpha_{\nu}(t)|$.
 Since $|e_{E, \nu}|=|e_{E, \nu}-\alpha_{\nu} + \alpha_{\nu} | \leq \|\bm{z}\| + |\alpha_{\nu}|, \forall t\geq 0$, we have  
	\begin{equation}
		|e_{E, \nu}| < \delta_{\bm{z}}  + \alpha_{\nu}, \forall t \geq 0.
		\label{eq_pe_alpha}
	\end{equation}
	By comparing \eqref{eq_alpha_rho} and \eqref{eq_pe_alpha}, one can conclude that $|e_{E, \nu}|< \rho_{\nu}(t), \forall t \geq 0$. This completes the proof. 
\end{proof}

\section{Reference Allocation}
This section proposes a reference allocation method to coordinate the quadcopter base and the Delta arm to track $\dot{\bm{p}}_{E,d}$ derived from \eqref{eq_control_preset}. In addition, physical constraints are considered in the proposed method to avoid solutions that may violate physical constraints, rendering them impractical for the aerial manipulator. The considered physical constraints of the aerial manipulator include position, velocity, and acceleration constraints.

Let $\bm{s}= [\bm{p}_B^T,\bm{p}_E^{D,T}]^T$ denote the state of the aerial manipulator, which consists of the position of the quadcopter base in $\varSigma_I$ and the position of the end-effector in $\varSigma_D$. Then, the time derivative of $\bm{p}_E$ is $\dot{\bm{p}}_E = \bm{J}\dot{\bm{s}} -[ \bm{R}_B\bm{p}_E^B]_{\times}\bm{\omega}$, where $\bm{J}=[\bm{I}_3,\bm{R}_B\bm{R}_D^B]$ is the Jacobian matrix, $\bm{I}_3$ is the $3\times 3$ identity matrix. 

To calculate the references of the quadcopter base and the Delta arm with the given desired end-effector position, the reference allocation problem is mathematically formulated as a QP problem:
\begin{equation}
	\begin{split}
		\min_{\dot{\bm{s}}} \quad &F(\dot{\bm{s}})=(\dot{\bm{p}}_E -\dot{\bm{p}}_{E,d})^T(\dot{\bm{p}}_E -\dot{\bm{p}}_{E,d}) + \dot{\bm{s}}^T\bm{W}\dot{\bm{s}}\\
		& \quad =\dot{\bm{s}}^T(\bm{J}^T\bm{J} + \bm{W})\dot{\bm{s}}\\
            & \quad \quad -2(\dot{\bm{p}}_{E,d}+[ \bm{R}_B\bm{p}_E^B]_{\times}\bm{\omega})^T\bm{J}\dot{\bm{s}}\\
		\text{s.t.}\quad 	& \dot{\bm{s}}_{\underline{p}} \leq \dot{\bm{s}} \leq \dot{\bm{s}}_{\overline{p}}, \dot{\bm{s}}_{\min} \leq \dot{\bm{s}} \leq \dot{\bm{s}}_{\max}, \dot{\bm{s}}_{\overline{v}} \leq \dot{\bm{s}} \leq \dot{\bm{s}}_{\overline{v}},
	\end{split}
	\label{eq_QP_problem}
\end{equation}
where $F(\dot{\bm{s}})$ is the cost function of the QP problem, $\bm{W}$ is a positive diagonal matrix,  $\dot{\bm{p}}_{E,d}$ is calculated by \eqref{eq_control_preset}, $\dot{\bm{s}}_{\underline{p}}$ and $\dot{\bm{s}}_{\overline{p}}$ are lower and upper approximate bounds of the actuated state position. $\dot{\bm{s}}_{\min}$ and $\dot{\bm{s}}_{\max}$ are the upper and lower bounds of $\dot{\bm{s}}$. $\dot{\bm{s}}_{\underline{v}}$ and $\dot{\bm{s}}_{\overline{v}}$ are lower and upper approximate bounds of the actuated state acceleration. 

In Equation \eqref{eq_QP_problem}, if position and acceleration constraints were directly described using terms of $\bm{s}$ and $\ddot{\bm{s}}$, the QP formulation would require 18 optimization variables. To reduce computational complexity, we approximate both position and acceleration constraints as expressions represented by velocity $\dot{\bm{s}}$, thereby decreasing the number of variables to 6. The analytical expressions for $\dot{\bm{s}}_{\underline{p}}$, $\dot{\bm{s}}_{\overline{p}}$, $\dot{\bm{s}}_{\underline{v}}$, and $\dot{\bm{s}}_{\overline{v}}$ can be derived from Section~VI-A of \cite{cao2023eso}. Let $\dot{\bm{s}}^*\in\mathbb{R}^{6}$ denote the optimal solution of QP problem \eqref{eq_QP_problem}. The constraints in \eqref{eq_QP_problem} ensure the optimal solution $\dot{\bm{s}}^*$ to satisfy the physical constraints of the aerial manipulator. The desired state vector, $\bm{s}_d=[\bm{p}_{B,d}^{T},\bm{p}_{E,d}^{D,T}]^T$, is obtained by integrating the optimal velocity solution, $\dot{\bm{s}}_d = \dot{\bm{s}}^*$. Consequently, the desired positions $\bm{p}_{B,d}$ and $\bm{p}_{E,d}^{D}$ can be derived from $\bm{s}_d$. The desired joint configuration $\bm{q}_d$ is then calculated using the inverse kinematics of the Delta arm. 

The QP problem is solved by an efficient solver, OSQP \cite{osqp}. To accelerate the QP problem solution, the result from the previous iteration is used as the initial value for the next iteration. Similar to CLIK, the proposed method also enables redundancy management in prioritized multi-task scenarios. This is achieved through configurable approaches in the QP formulation, either by adding constraint conditions for hard priority enforcement or incorporating additional terms in the cost function for soft priority adjustment \cite{lunni2017nonlinear}.

\begin{figure*}[t]
	\centering
	\includegraphics[width=1.0\linewidth]{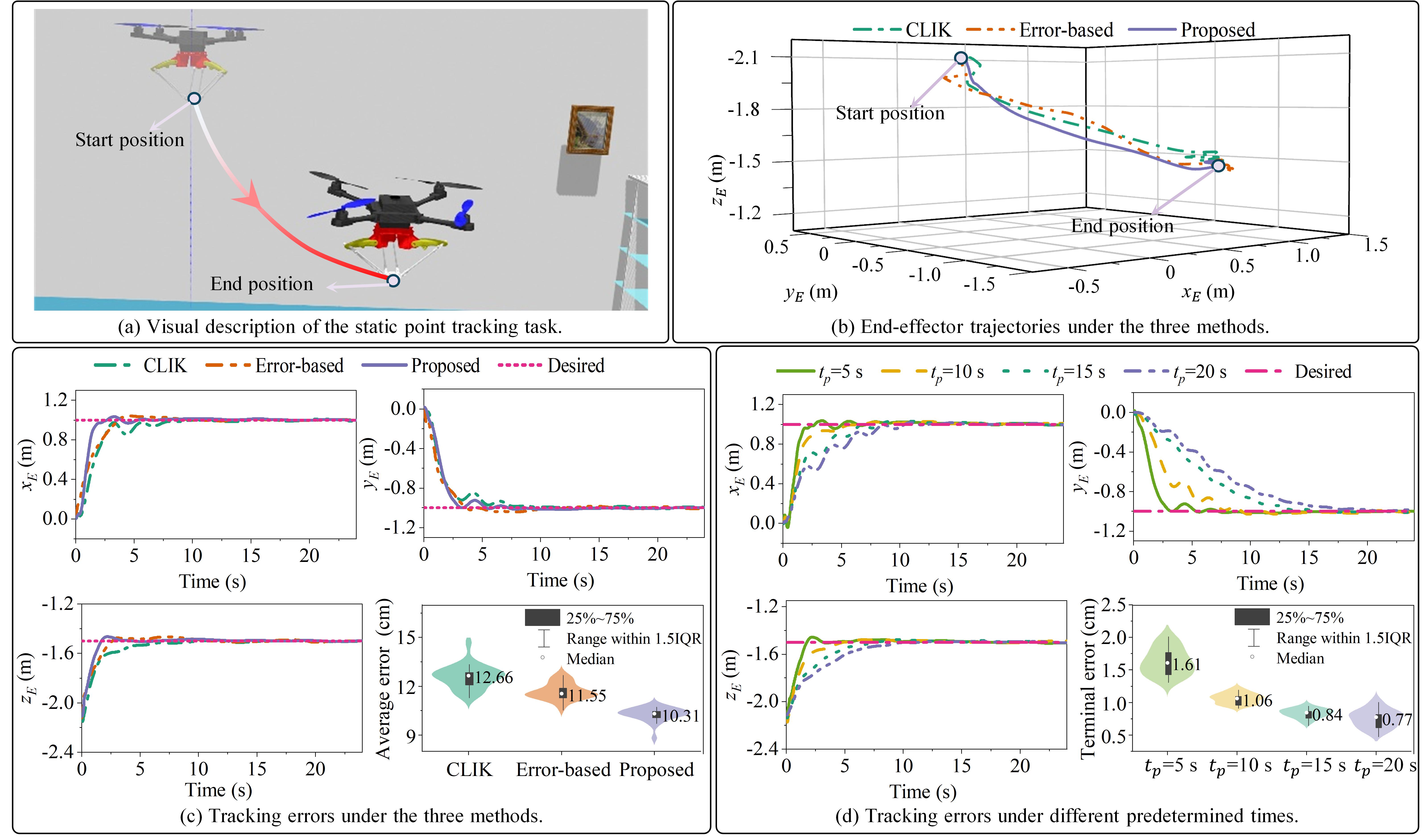}
	\caption{Experiment results of the static point tracking example.}
	\label{fig_comparison_staic_point_tracking}
\end{figure*}

\begin{figure*}[t]
	\centering
	\includegraphics[width=1.0\linewidth]{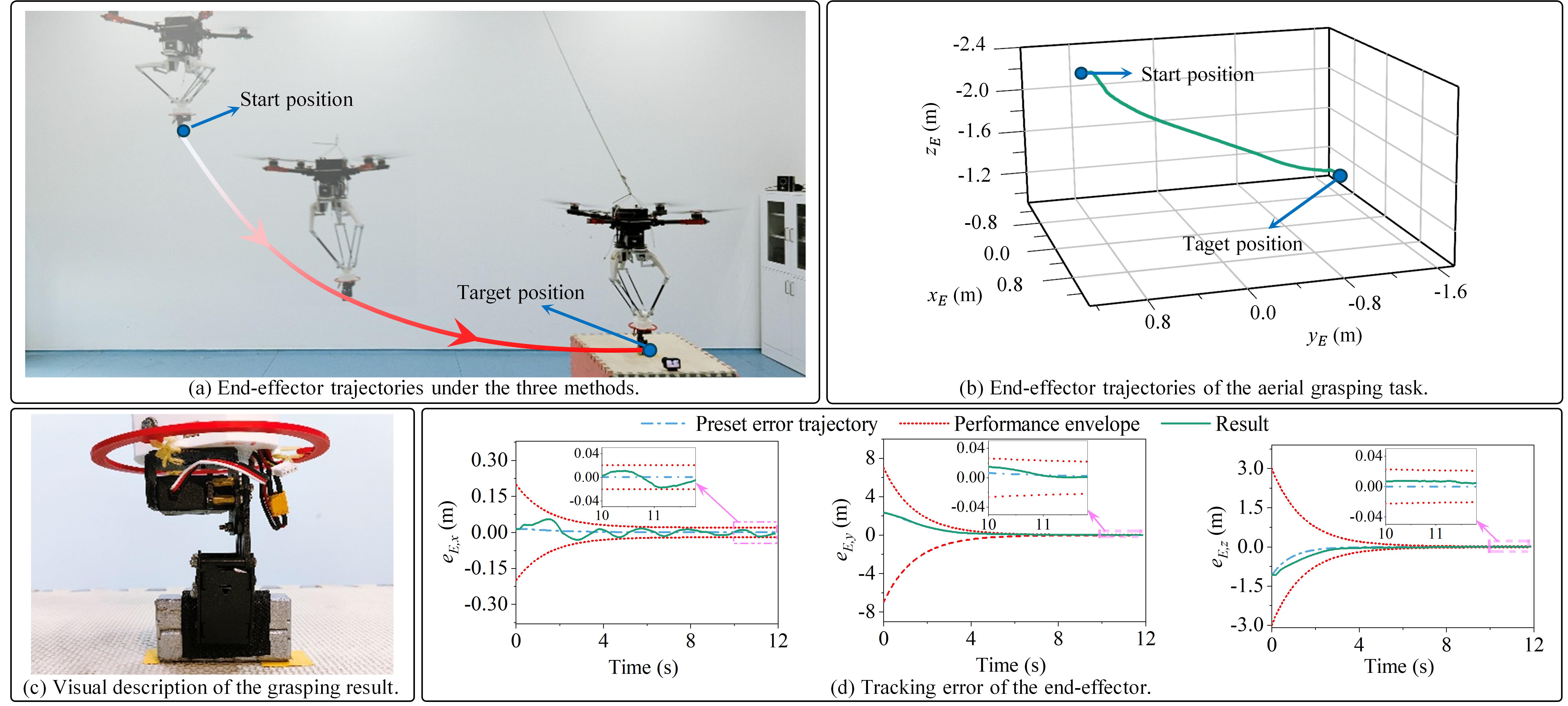}
	\caption{Experiment results of the aerial grasping example.}
	\label{fig_grasping_result}
\end{figure*}

\begin{figure*}[t]
	\centering
	\includegraphics[width=1.0\linewidth]{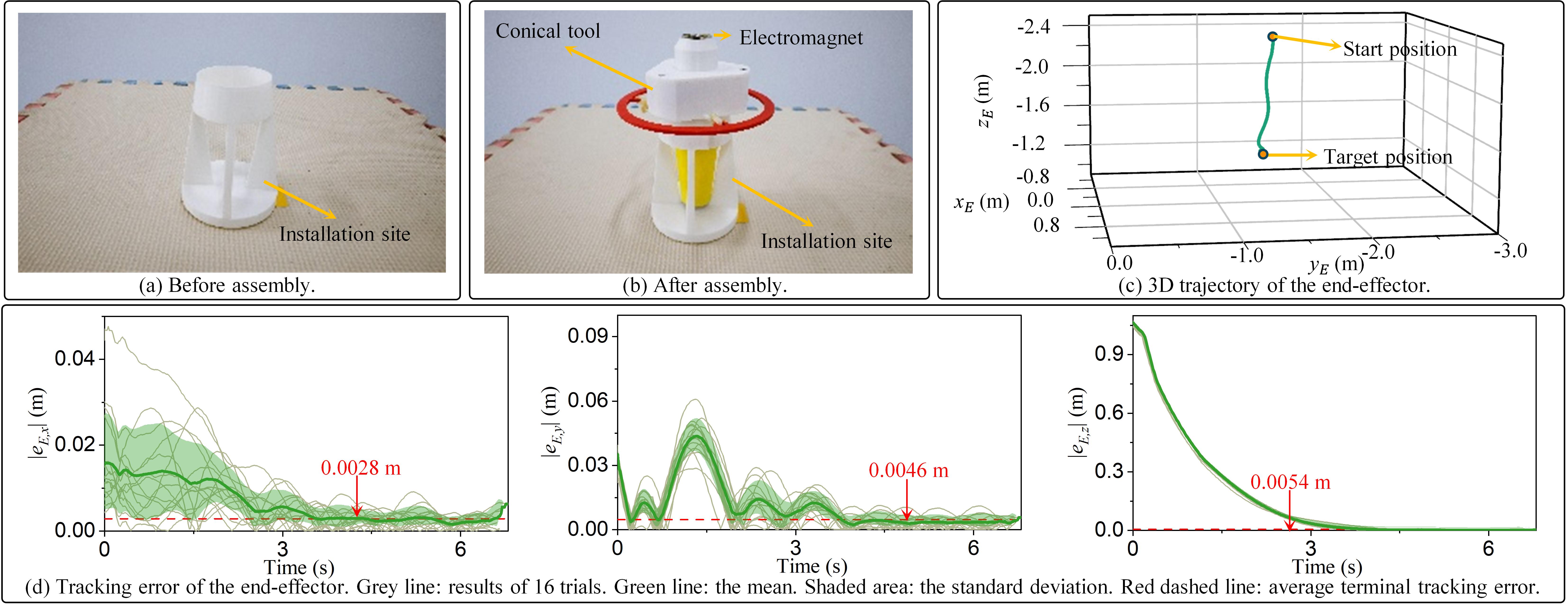}
	\caption{Experiment results of the peg-in-hole assembly example.}
	\label{fig_peg_in_hole_result}
\end{figure*}

\section{Experimental Verification}\label{sec_6}

This section presents experimental results to verify the effectiveness of the proposed algorithms. First of all, we describe the experimental setup. The motor-to-motor distance of the quadcopter base is $0.65$~m. The mass of the quadcopter (including a battery) is $3.60$~kg. The Delta arm consists of a mounting base ($0.56$~kg) and a movable robotic arm ($0.44$~kg). The proposed tracking controller and the inverse kinematics of the Delta arm run on an onboard Intel NUC i7 computer with ROS. An ESO-based nonlinear controller is used as the flight controller for the quadcopter, running on a Pixhawk 4 autopilot \cite{cao2023eso}. The controller inputs include the desired position $\bm{p}_{B,d}$ and the desired yaw angle $\psi_{d}$. Real-world experiments utilized a Vicon motion capture system to measure the quadcopter base and end-effector positions with high accuracy (see Fig.~\ref{fig_structureController}c). 

In all the experiments, we use the same set of control gains: $\bm{\Lambda}=\text{diag}([0.2, 0.2, 0.2])$, $\bm{K}=\text{diag}([ 1.2, 1.2, 1.2])$, $\delta_E = 0.01$. Parameter $\bm{\rho}_{0}$ is configured according to the initial tracking error of each experiment. Parameter $\bm{\rho}_{\infty}$ is set as [0.02, 0.02, 0.02] for all experiments in this paper.

\subsection{Example 1: Static Point Tracking}
The goal of this experiment is to rapidly maneuver the end-effector from a start position $[0,0,-2.1]$~m to an end position $[1,-1,-1.5]$~m. This task demonstrates the accuracy and responsiveness of the proposed control algorithm. We simulate this task because many real-world applications (e.g., aerial grasping and peg-in-hole assembly) can be decomposed into point-to-point motion problems. Thus, achieving high-precision static point tracking is critical for aerial manipulators. The simulation was performed in Gazebo 11 (Fig.~\ref{fig_comparison_staic_point_tracking}(a)), considering sensor noises (GPS, gyroscope, and accelerometer). All sensor noises are modeled as zero-mean Gaussian processes with standard deviations of 0.01 cm for GPS noise, 0.02 rad/s$^2$ for gyroscope noise, and 0.04 m/s$^2$ for accelerometer noise.

To demonstrate the advantages of the proposed algorithm, we compare the proposed method with two methods: the CLIK method and the error-based method. The CLIK method is formed by combining the CLIK feedback strategy from \cite{baizid2017behavioral} with our proposed QP-based reference allocation. The error-based method directly utilizes the end-effector tracking error in conjunction with MPC. In the proposed method, parameters $\bm{\rho}_{0}$ and $\bm{c}$ are set as [2.0, 2.0, 1.2] and [0.25, 0.25, 0.25], respectively. Fig.~\ref{fig_comparison_staic_point_tracking}(b) and (c) present the comparative results. Twenty-one simulations per method were conducted to evaluate the tracking performance of the three methods. Fig.~\ref{fig_comparison_staic_point_tracking}(c) presents the median values of average tracking errors for each method. The proposed method demonstrated superior performance with a median error of 10.31 cm (standard deviation: 0.45~cm), representing 22.8\% and 12.0\% reductions compared to the CLIK method (12.66 cm, standard deviation: 0.75~cm) and error-based method (11.55 cm, standard deviation: 0.52~cm), respectively. The reason for the relatively large average tracking error is that the initial point is far from the target position. Furthermore, it demonstrated faster convergence with a median time of 2.4s, outperforming both the CLIK (4.8s) and error-based (3.2s) methods.

To further investigate the influence of the preset time $t_p$ on the results, we perform simulations with different preset time values: 5s, 10s, 15s, and 20s. The simulation results are shown in Fig.~\ref{fig_comparison_staic_point_tracking}(d). It can be observed that increasing the preset time reduces the tracking error in the terminal phase. To conduct a more detailed analysis, we performed 21 simulations for each of these four values. We define the terminal tracking error as the average error within five seconds after $||\bm{e}_E||\leq ||\bm{\rho}_{\infty}||$. The statistical results show that by increasing the preset time from 5 to 20s, the terminal tracking error decreases from 1.61 cm to 0.77 cm. 

\subsection{Example 2: Aerial Grasping}
The goal of this experiment is to grasp a target object on a pillar. In the experiment, $t_p$, $\bm{\rho}_{0}$, and $\bm{c}$ are set as 10s, [0.1, 5.0, 3.0], and [0.15, 0.15, 0.15], respectively. The start position of the end-effector is set as [0.01, 0.99, -2.42]~m. The position of the target object is set as [0, -1.36, -1.34]~m. To achieve target object grasping, a rigid two-finger gripper is mounted on the end effector (see Fig.~\ref{fig_structureController}d). The visual description of the experiment is shown in Fig~\ref{fig_grasping_result}(a). 

The experiment demonstrates that our proposed method is capable of achieving aerial grasping. Fig~\ref{fig_grasping_result}(b-d) illustrates the results of the experiment. The 3D trajectory of the end effector is shown in Fig~\ref{fig_grasping_result}(b). Fig~\ref{fig_grasping_result}(c) gives the visual description of the grasping result. The tracking errors of the end-effector are given in Fig~\ref{fig_grasping_result}(d). The results show that the tracking errors in all three directions are contained within the performance envelope. The terminal tracking error of the proposed method in this experiment is 1.26 cm (standard deviation: 0.81 cm). The total task completion time was 11.8s.

\subsection{Example 3: Peg-in-hole Assembly}
The goal of this experiment is to perform a peg-in-hole assembly of a conical tool into an installation site. An electromagnet is mounted directly above the tool, enabling controllable tool attachment and detachment (see Fig.~\ref{fig_structureController}d). In the experiment, $t_p$, $\bm{\rho}_{0}$, and $\bm{c}$ are set as 5s, [0.1, 0.1, 2.0], and [0.3, 0.3, 0.3], respectively. The start position of the end effector is [0.04, -1.47, -2.40]~m, and the position of the installation site is [0.04, -1.37, -1.34]~m. The visual description of the experiment is shown in Fig.~\ref{fig_pig_in_hole}. 

The experiment demonstrates that our proposed method is capable of achieving aerial peg-in-hole assembly. Before assembly, the installation site is mounted on a pillar (Fig.~\ref{fig_peg_in_hole_result}(a)). After assembly, the tool is inserted into the installation site (Fig.~\ref{fig_peg_in_hole_result}(b)). The tracking results of the end-effector are shown in Fig.~\ref{fig_peg_in_hole_result}(c) and (d). The 16 experimental trials exhibited durations ranging from 6.2 to 6.8s, with a mean completion time of 6.5s. The average terminal tracking errors of the proposed method in the three directions are 0.28 cm, 0.46 cm, and 0.54 cm, respectively. The overall average terminal tracking error in three directions was 0.89~cm, and the standard deviation was 0.64~cm.

\section{Conclusion}\label{sec_7}
This paper proposes a novel kinematic tracking control method composed of preset trajectory tracking control and QP-based reference allocation. Its performance was validated through three experiments. In the static point tracking experiment, comparative results demonstrate that the proposed method achieves a 22.8\% reduction in average tracking error relative to the CLIK method and a 12.0\% reduction compared to the error-based method. Furthermore, it is observed that adjusting the preset time $t_p$ can effectively reduce the terminal tracking error of the end-effector. The experimental results from both aerial grasping and peg-in-hole assembly tasks demonstrate that the proposed algorithm is capable of being effectively applied to real-world aerial manipulation scenarios. Future research can focus on the following directions: (1) employing disturbance estimation methods (e.g., offline-trained neural networks \cite{zhao2025RLBook}) to estimate $\bm{\Delta}$ for improving the system's robustness against dynamic coupling and wind disturbances, and (2) implementing redundancy management strategies to enhance task completion capability.

\bibliography{myOwnPub} 
\bibliographystyle{ieeetr}
	

\end{document}